\documentclass[manuscript,screen]{acmart}

\acmJournal{TIST}
\acmYear{2026}
\acmVolume{0}
\acmNumber{0}
\acmArticle{0}
\acmMonth{0}
\acmDOI{}                            
\setcopyright{none}                  

\newtheorem{assumption}{Assumption}

\begin{document}

\title{Agentic Governance and Adversarial Verification for Policy-Constrained LLM Healthcare Appeal Generation}

\author{Harshil Lodhiya}
\orcid{0009-0005-5798-9176}
\affiliation{%
  \position{Chief Software Architect}
  \institution{Sliced Health}
  \city{Woodstock}
  \state{GA}
  \country{United States}}
\email{hlodhiya@slicedhealth.com}

\author{Alex McManus}
\orcid{0009-0002-8292-0793}
\affiliation{%
  \position{Chief Technology Officer}
  \institution{Sliced Health}
  \city{Woodstock}
  \state{GA}
  \country{United States}}
\email{amcmanus@slicedhealth.com}

\author{Reese Walker}
\affiliation{%
  \position{Chief Product Officer}
  \institution{Sliced Health}
  \city{Woodstock}
  \state{GA}
  \country{United States}}
\email{rwalker@slicedhealth.com}

\renewcommand{\shortauthors}{Lodhiya et al.}

\begin{abstract}
Claim denial management is one of the costliest operational bottlenecks in United States healthcare, driving an estimated \$260 billion per year in administrative overhead and delayed reimbursements~\cite{morse2024denials}. Large Language Models (LLMs) and Retrieval-Augmented Generation (RAG) can produce fluent clinical text, but single-agent architectures are poorly matched to high-stakes administrative healthcare: they may introduce unsupported clinical details, lose the logical structure of hierarchical payer policy, and provide only weak post-hoc evidence that generated assertions are admissible.

We propose the Agentic Governance and Adversarial Verification Framework (AGVF)---a multi-agent architecture for medical-necessity appeal generation under explicit policy and evidence constraints. AGVF models appeal synthesis as a Constrained Markov Decision Process (CMDP) over five agent roles: policy formalization, evidence retrieval, gap analysis, adversarial critique, and gated synthesis.

We prove that, under explicit assumptions, refinement over a fixed policy constraint graph monotonically reduces a formal evidence-deficiency potential and terminates either with a complete satisfying policy frontier or with a localized residual evidence gap. We also define a deterministic citation-grounding gate that prevents assertions without admissible evidence citations from entering shared state. We provide a dependency-light reference implementation and validate it on 1{,}000 synthetic appeal cases whose diagnosis, procedure, and charge distributions are parameterized from aggregate, de-identified public hospital-discharge data. The validation confirms implementation fidelity for the formal invariants: zero citation-grounding violations across all AGVF cases and monotone deficiency reduction in every episode; ablating the deterministic gate raises the grounding-violation rate to 100\%. The study uses no real patient records and does not measure clinical efficacy or payer-overturn outcomes. AGVF thus contributes a theory-backed agentic architecture, a controlled synthetic benchmark, and a verified reference implementation for policy-constrained LLM generation in a high-stakes healthcare workflow.
\end{abstract}



\maketitle

\begin{center}
  \textit{This manuscript is under review at ACM Transactions on Intelligent Systems and Technology (TIST).}
\end{center}

\section{Introduction}

\subsection{Background and Domain Challenges}
The financial viability of United States healthcare organizations depends in large part on navigating complex and often adversarial claim adjudication systems. Health insurers process approximately three billion medical claims each year, of which between 9\% and 15\% are initially denied upon submission~\cite{premier2025adjudication,premier2025trend}. Claim denials fall into several distinct categories---ranging from administrative formatting errors to substantive clinical disagreements regarding Medical Necessity (e.g., Claim Adjustment Reason Code CO-50) and Prior Authorization Non-Compliance (e.g., CO-197).

Overturning an improper medical necessity denial requires a clinical appeal package that synthesizes three disparate, non-standardized knowledge streams:
\begin{itemize}
  \item \textbf{Unstructured EHR Data:} Longitudinal patient charts, laboratory panels, imaging reports, nursing notes, and physician attestations.
  \item \textbf{Complex Payer Policy Logic:} Dynamic, hierarchical coverage rules comprising CMS National Coverage Determinations (NCDs), Local Coverage Determinations (LCDs), InterQual criteria, MCG guidelines, and commercial payer-specific policy bulletins.
  \item \textbf{Regulatory Compliance Frameworks:} Strict timelines and statutory requirements mandated by ERISA, CMS, and state-level insurance commissioner regulations.
\end{itemize}

Today, appeals are prepared manually by specialized nurses and physician advisors. The cost of manually processing a single complex appeal ranges from roughly \$25 to upward of \$120 per claim, depending on complexity and the clinical labor involved~\cite{premier2025trend,acdis_changehealthcare,cofactor2025appeals}. Due to resource constraints, healthcare organizations appeal fewer than 50\% of eligible denied claims, absorbing significant uncompensated care costs~\cite{cofactor2025appeals}.

\subsection{Failure Modes of Prevalent AI Approaches}
Large Language Models (LLMs) have drawn attention as a path toward automating appeal generation. In practice, however, deploying single-prompt LLMs or basic Retrieval-Augmented Generation (RAG) pipelines in this setting exposes several structural failure modes (Figure~\ref{fig:naive}).

\begin{figure}[tbp]
  \centering
  \includegraphics[width=\linewidth]{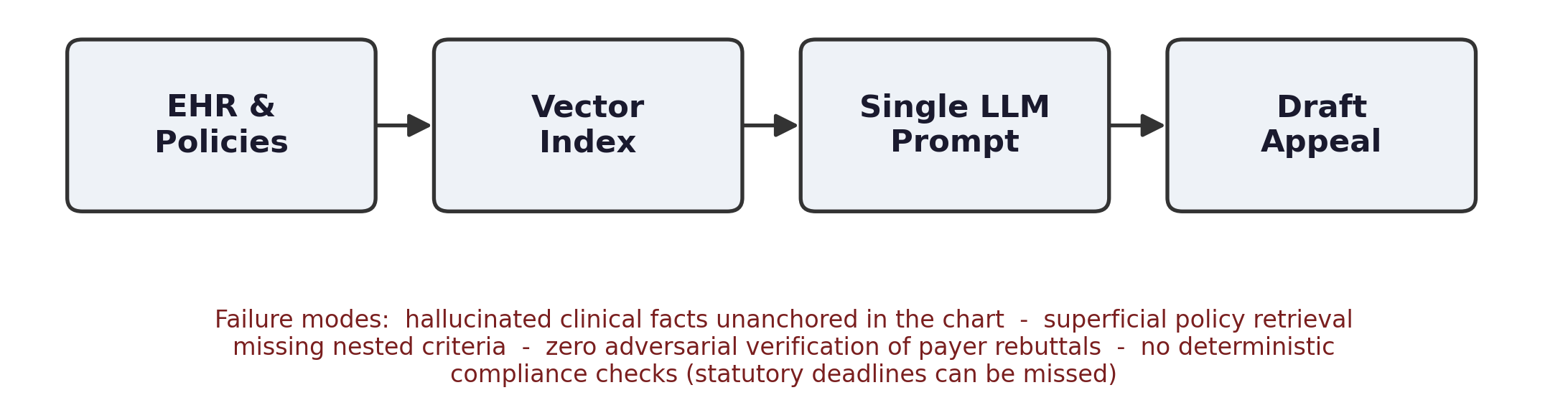}
  \caption{Naive single-agent RAG pipeline and its structural failure modes.}
  \Description{Block diagram of a naive retrieval-augmented generation pipeline showing how retrieval, generation, and post-hoc review can lead to hallucinated claims, fragmented policy context, weak adversarial review, and limited auditability.}
  \label{fig:naive}
\end{figure}

\noindent\textbf{Epistemic Hallucination and Unanchored Claims.} Generative models can extrapolate clinical details (e.g., asserting a patient failed a conservative therapy that was never documented), exposing providers to legal liability and fraud allegations under the False Claims Act.

\noindent\textbf{Context Fragmentation in Policy Retrieval.} Payer policies contain nested conditional logic (e.g., ``step therapy required with Class A drug for six weeks unless contraindicated by eGFR less than 30~mL/min''). Standard semantic retrieval alone may fail to preserve these logical dependencies, retrieving fragments that are topically relevant but procedurally insufficient.

\noindent\textbf{Lack of Adversarial Foresight.} A standard text generator creates passive summaries. It fails to anticipate the rebuttal strategies employed by insurance company medical directors trained to deny appeals based on minor evidentiary omissions.

\noindent\textbf{Absence of Deterministic Audit Trails.} Healthcare compliance requires non-probabilistic controls over what enters payer-facing text. Stochastic LLM outputs alone cannot guarantee that each generated assertion is tied to an admissible evidence source.

\paragraph{Evaluation constraint.} Real denial packets combine protected health information, payer contracts, and operationally sensitive adjudication outcomes, making public benchmarking difficult. We therefore evaluate AGVF on a controlled synthetic benchmark whose surface distributions are parameterized from public de-identified discharge aggregates, while its evidence obligations and oracle labels are generated synthetically. This design does not establish clinical efficacy; rather, it tests whether the agentic control structure enforces the formal invariants it claims to enforce under reproducible conditions.

\subsection{Theoretical Formulation: Appeal Optimization as a CMDP}
To close the gap between what generative models can produce and what regulatory compliance demands, we formulate healthcare appeal management as a Constrained Markov Decision Process (CMDP) defined by the six-tuple:
\begin{equation}
  \mathcal{M} = \left( \mathcal{S}, \mathcal{A}, \mathcal{P}, \mathcal{R}, \mathcal{C}, \gamma \right)
\end{equation}
where $\mathcal{S}$ represents the state space of the appeal dossier, containing the set of verified patient facts $F \subset \mathcal{F}_P$, extracted policy constraints $K \subset \mathcal{K}_I$, and the current draft arguments; $\mathcal{A}$ is the action space executed by specialized software agents (e.g., querying evidence pools, synthesizing counter-arguments, applying deterministic gates); $\mathcal{P}: \mathcal{S} \times \mathcal{A} \to \Delta(\mathcal{S})$ defines the state transition dynamics across the agentic execution trajectory; and $\mathcal{R}(\cdot)$ is the objective reward function estimating the probability of claim overturn:
\begin{equation}
  \mathcal{R}(s) = P(\text{Overturn} \mid s).
\end{equation}
$\mathcal{C}: \mathcal{S} \times \mathcal{A} \to \mathbb{R}^{k}$ defines a set of $k$ admissibility costs, including unsupported-assertion and regulatory-template metrics, and $\gamma \in (0,1]$ is the discount factor over iterative refinement steps.

The idealized goal is a cooperative policy $\pi^{*}$ that solves:
\begin{equation}
  \max_{\pi} \; \mathbb{E}_{\tau \sim \pi} \!\left[ \sum_{t=0}^{T} \gamma^{t} \mathcal{R}(s_t, a_t) \right]
  \quad \text{s.t.} \quad
  \mathbb{E}_{\tau \sim \pi} \!\left[ \sum_{t=0}^{T} \mathcal{C}_i(s_t, a_t) \right] \le d_i \;\; \forall i \in \{1,\dots,k\}
\end{equation}
where $d_i$ represents the maximum allowable threshold for regulatory omissions ($d_{\text{regulatory}}=0$ when a concrete checklist is specified) and unsupported generated assertions ($d_{\text{unsupported}}=0$).

\subsection{Proposed Agentic Architecture: AGVF}
To solve this CMDP we introduce AGVF, which decomposes the problem into a multi-agent graph where each agent has a specialized role (Figure~\ref{fig:agvf}).

\begin{figure}[tbp]
  \centering
  \includegraphics[width=0.72\linewidth]{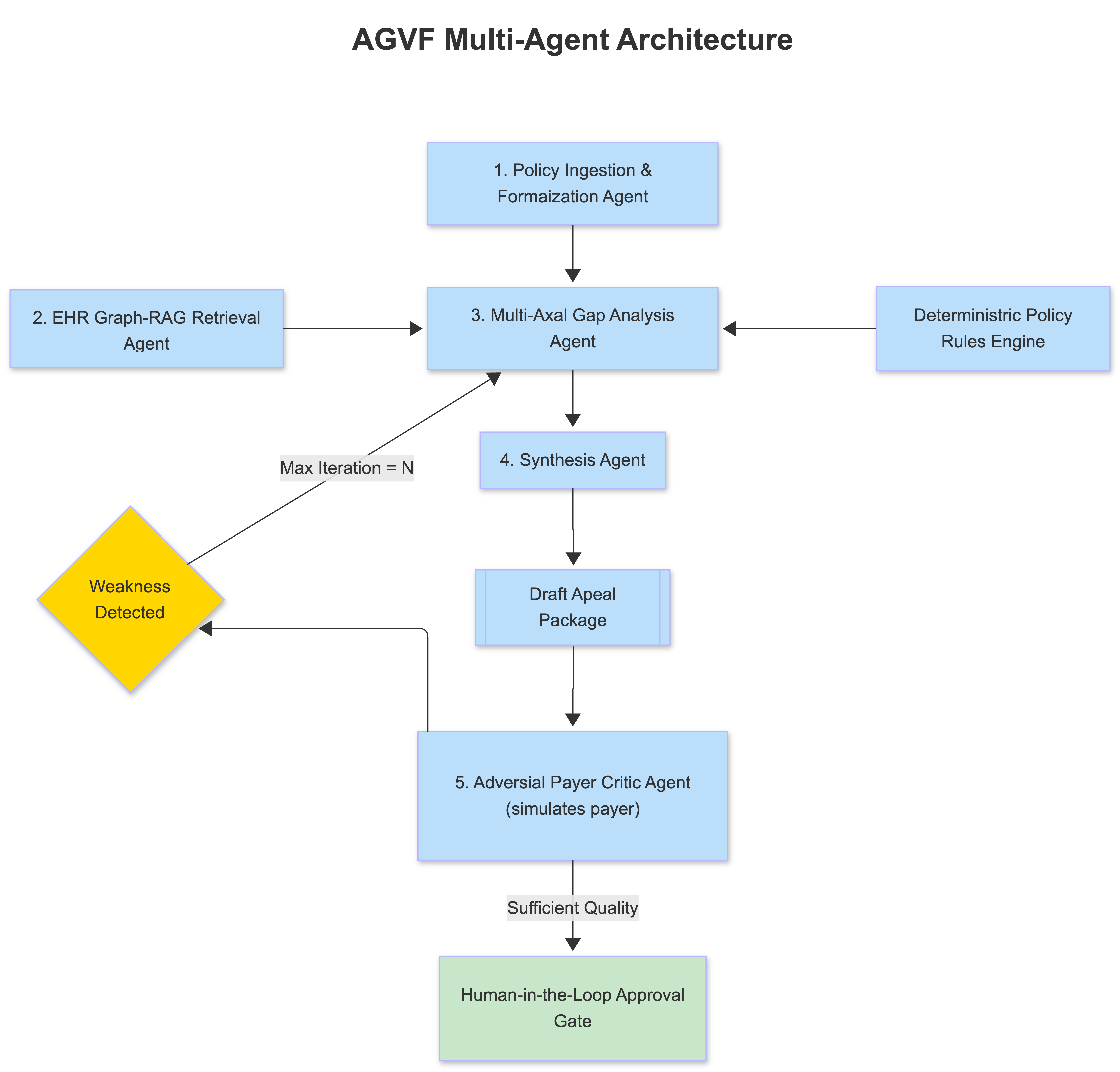}
  \caption{AGVF multi-agent architecture and adversarial refinement loop.}
  \Description{Architecture diagram showing policy ingestion, evidence retrieval, gap analysis, adversarial critique, and gated synthesis coordinated through an iterative refinement loop.}
  \label{fig:agvf}
\end{figure}

\noindent\textbf{Policy Ingestion and Formalization Agent.} Represents payer rules as an executable policy constraint graph $G_K=(V_K,E_K)$ with stable identifiers for atomic obligations. In the reference implementation these graphs are hand-authored JSON artifacts; automated policy-to-graph ingestion is future work.

\noindent\textbf{Evidence Retrieval Agent.} Queries a case evidence pool for facts corresponding to currently open policy obligations. The reference implementation uses exact-key retrieval for reproducible synthetic evaluation and exposes a semantic-retrieval interface for real LLM-backed experiments.

\noindent\textbf{Multi-Axial Gap Analysis Agent.} Executes formal logic matching between $G_K$ and patient evidence $F_P$ to compute missing clinical prerequisites or identify diagnostic equivalents.

\noindent\textbf{Adversarial Payer Critic Agent.} Evaluates the draft from an opposing-adjudicator perspective and returns obligations that remain contestable. In the deterministic validation reported here, this role is instantiated as a sound gap critic; measuring a real LLM critic is an explicit extension.

\noindent\textbf{Gated Synthesis Agent.} Assembles draft assertions only after deterministic citation-grounding checks, then hands the residual evidence gaps and draft assertions to a Human-in-the-Loop (HITL) reviewer.

\subsection{Key Contributions of This Paper}
This paper makes five contributions:
\begin{itemize}
  \item \textbf{Theoretical Framework.} We formalize healthcare appeal generation as a CMDP under hard evidence-admissibility constraints, moving beyond heuristic prompt engineering toward bounded multi-agent refinement.
  \item \textbf{Agentic Architecture (AGVF).} We present an agentic control architecture that separates policy representation, evidence retrieval, formal gap analysis, adversarial critique, and gated synthesis.
  \item \textbf{Convergence Analysis.} We prove that refinement over a static policy graph with append-only evidence monotonically reduces a formal deficiency potential and terminates in either a complete satisfying frontier or a residual evidence gap.
  \item \textbf{Synthetic Benchmark.} We construct a SPARCS-parameterized synthetic appeal benchmark with hand-formalized policy graphs and oracle obligation-satisfaction labels, enabling reproducible evaluation without exposing patient records.
  \item \textbf{Reference Implementation and Validation.} We provide a dependency-light reference implementation and validate on 1{,}000 synthetic cases that the implementation upholds its formal invariants: zero citation-grounding violations and monotone convergence in all AGVF episodes. An ablation confirms the deterministic citation gate is load-bearing (its removal drives the grounding-violation rate to 100\%). Clinical efficacy evaluation on clinician-adjudicated data is left to future work.
\end{itemize}

\subsection{Structure of the Remainder of the Paper}
Section~2 reviews related work. Section~3 develops the formal theory---constraint graphs, the refinement protocol, and the convergence and grounding proofs. Section~4 describes the system implementation. Section~5 validates the reference implementation on synthetic cases. Section~6 discusses limitations and future directions.

\section{Related Work}
We situate AGVF relative to five areas: retrieval, multi-agent reasoning, constrained decision-making, iterative self-correction, and clinical NLP.

\subsection{Retrieval-Augmented and Graph-Structured Retrieval}
Retrieval-Augmented Generation pairs a language model with a retriever so that outputs can cite an external corpus~\cite{lewis2020rag}. Standard chunk-based implementations work poorly on the long, deeply nested documents typical of payer coverage policy. RAPTOR addresses this by recursively clustering and summarizing text into a tree that supports retrieval at multiple granularities~\cite{sarthi2024raptor}; Graph-RAG goes further and structures the corpus as a traversable knowledge graph~\cite{edge2024graphrag}. AGVF is compatible with hierarchical and graph-structured retrieval, but the reference implementation deliberately uses exact-key synthetic retrieval so that the theoretical invariants can be evaluated without conflating them with retriever quality.

\subsection{Multi-Agent and Debate-Style LLM Reasoning}
Having multiple LLM instances debate over several rounds improves factual accuracy~\cite{du2023multiagent}, and Tree-of-Thought search lets a model explore alternative reasoning paths before committing~\cite{yao2023tot}. Neither method, however, offers convergence or compliance guarantees; even the debate literature acknowledges that free-form debate may not converge. AGVF restricts critique to a fixed policy-constraint board and pairs refinement with deterministic gates, which is what produces the guarantees in Section~\ref{sec:theory}. The present experiments instantiate the critic deterministically; LLM-based adversarial critique is part of the system design and future empirical extension.

\subsection{Constrained Sequential Decision-Making}
A Constrained Markov Decision Process~\cite{altman1999cmdp} separates a reward to maximize (here, policy coverage as a surrogate for overturn probability) from hard costs to bound (unsupported assertions and regulatory omissions). What AGVF adds is that selected admissibility constraints are \emph{deterministically} enforced by gates rather than left as expected penalties.

\subsection{Self-Correction and Iterative Refinement in LLMs}
A parallel line of work improves LLM outputs through iterative self-critique rather than a single forward pass. Self-Refine has a single model generate, critique, and revise its own output in a feedback loop, improving quality without additional training~\cite{madaan2023selfrefine}. Reflexion equips language agents with verbal self-reflection stored in episodic memory to learn from prior failures~\cite{shinn2023reflexion}. These methods establish that a generate--critique--revise loop measurably improves LLM output, which motivates AGVF's refinement structure. AGVF differs in two decisive respects. First, its critique is constrained to explicit policy obligations rather than open-ended preference feedback. Second, self-refinement methods offer no termination or admissibility guarantee: the loop may oscillate, plateau, or ``improve'' toward a fluent but unsupported output. AGVF replaces open-ended self-feedback with refinement bounded by a formal constraint graph and a deficiency potential, yielding the monotone convergence of Theorem~\ref{thm:term}, and replaces soft self-critique with a deterministic citation-grounding gate, yielding the admissibility invariant of Lemma~\ref{lem:ground}.

\subsection{LLMs for Clinical and Administrative Healthcare Text}
LLMs are increasingly applied to clinical text---question answering, documentation drafting, coding assistance---but the central concern remains factual reliability. In a regulatory setting, a fluent but unsupported statement is not a minor error; it is a liability that can implicate fraud statutes when it enters a payer-facing document. Most deployed systems address this through post-hoc human review. AGVF takes a different approach: admissibility is checked on the write path. No assertion enters the draft unless it cites verified evidence in the record, so the human reviewer audits a document whose claims are citation-grounded by construction. To our knowledge, framing medical-necessity appeal generation as constrained multi-agent refinement over a policy graph with deterministic admissibility gates is novel to this work.

\section{Theoretical Methodology}\label{sec:theory}
This section develops the formal core of AGVF: the objects agents manipulate (\S\ref{ssec:prelim}), the interaction protocol (\S\ref{ssec:protocol}), the refinement game (\S\ref{ssec:game}), and the two invariants that separate AGVF from a naive pipeline---monotone deficiency reduction (\S\ref{ssec:conv}) and deterministic citation grounding (\S\ref{ssec:zero})---followed by a constrained optimality characterization (\S\ref{ssec:opt}).

\subsection{Formal Preliminaries and Notation}\label{ssec:prelim}
\paragraph{Policy constraint graph.} The Policy Ingestion and Formalization Agent compiles, or in the reference implementation loads, a payer coverage policy (an LCD, NCD, or commercial bulletin) into a \emph{policy constraint graph} $G_K=(V_K,E_K)$. Each leaf node $v \in L(V_K) \subseteq V_K$ carries an atomic clinical or procedural predicate $p_v$ (e.g., ``documented eGFR $<30$~mL/min'' or ``$\ge 6$ weeks of Class-A step therapy recorded''). Internal nodes are logical connectives drawn from $\{\wedge,\vee,\neg\}$; the ``unless'' clauses of \S1.2 are encoded as guarded Boolean structure. A distinguished root $r \in V_K$ denotes the top-level medical-necessity criterion, so the graph defines a Boolean formula $\Phi_G$ over the leaf predicates.

\paragraph{Evidence and anchoring.} Let $\mathcal{F}_P$ be the finite set of \emph{verified} patient facts available to the episode; each $f \in \mathcal{F}_P$ has a source locator and is admissible only if it can be traced to a source document. An \emph{anchoring relation} $\varphi$ records admissible evidence: $(p_v,f)\in\varphi$ iff fact $f$ substantiates predicate $p_v$. For an evidence subset $F \subseteq \mathcal{F}_P$, the satisfaction indicator $\mathrm{sat}_F(v)$ equals $1$ for a leaf $v$ iff $\exists f \in F$ with $(p_v,f)\in\varphi$, extended to internal nodes by the Boolean semantics of their connective. Write $\mathrm{SAT}_F(G_K)=\mathrm{sat}_F(r)$.

\paragraph{Dossier state.} A dossier state $s=(F_s,G_K,\delta_s)\in\mathcal{S}$ bundles the evidence gathered so far $F_s \subseteq \mathcal{F}_P$, the fixed policy graph, and the current draft $\delta_s$. The draft is a finite set of \emph{assertions} $\{c_1,\dots,c_m\}$; each assertion $c$ carries its cited support $\varphi(c)\subseteq\mathcal{F}_P$.

\paragraph{Deficiency.} A \emph{satisfying frontier} is a minimal set $A \subseteq L(V_K)$ of leaves whose joint satisfaction forces $\mathrm{SAT}(G_K)=1$; let $\mathcal{A}(G_K)$ collect all such frontiers. With materiality weights $w_v>0$, the \emph{deficiency} of a state is the cost of the cheapest still-unmet frontier:
\begin{equation}\label{eq:deficiency}
  D(s) = \min_{A \in \mathcal{A}(G_K)} \sum_{v \in A \,:\, \mathrm{sat}_{F_s}(v)=0} w_v,
\end{equation}
with $D(s)=0$ exactly when $\mathrm{SAT}_{F_s}(G_K)=1$. Thus $D(s)$ measures how far current evidence is from establishing medical necessity along the cheapest viable argument.

\subsection{Multi-Agent Interaction Protocol}\label{ssec:protocol}
AGVF realizes the CMDP policy $\pi$ as a fixed composition of five operators on $\mathcal{S}$, iterated under the control of the Adversarial Payer Critic. Writing $s_t$ for the state at refinement round $t$:
\begin{itemize}
  \item \textbf{Ingestion} $\mathcal{I}$ (once, $t=0$): policy specification $\mapsto G_K$, establishing $\Phi_G$ as an executable Boolean constraint graph. The current implementation uses hand-formalized graphs; automated extraction is outside the reported validation.
  \item \textbf{Retrieval} $\mathcal{R}\colon (G_K,s_t)\mapsto F'$. The Evidence Retrieval Agent issues targeted queries for currently open leaves and returns newly verified facts $F'$. Hierarchical and graph-structured retrieval are compatible back-ends, but the synthetic validation uses exact-key retrieval over oracle-labeled evidence pools.
  \item \textbf{Gap analysis} $\mathcal{G}\colon s_t \mapsto \Delta(s_t)$, the open-obligation set on the cheapest frontier, i.e.\ the argmin frontier in \eqref{eq:deficiency} restricted to its unsatisfied leaves.
  \item \textbf{Synthesis} $\mathcal{S}_{\!yn}\colon s_t \mapsto \delta_{t+1}$, drafting assertions that cite anchored evidence for satisfied obligations, subject to the deterministic gate of \S\ref{ssec:zero}.
  \item \textbf{Adversarial critique} $\mathcal{K}\colon \delta_{t+1}\mapsto W_t \subseteq \Delta(s_t)$. The Payer Critic returns the subset of open obligations it can still contest. In a real LLM instantiation this operator may use role-specialized debate or Tree-of-Thought-style search~\cite{du2023multiagent,yao2023tot}; in the reported deterministic validation it is implemented as a sound gap critic.
\end{itemize}
The loop repeats $\mathcal{R}\to\mathcal{G}\to\mathcal{S}_{\!yn}\to\mathcal{K}$ while $W_t\neq\emptyset$ and $t<N$ (the max-iteration budget). Each round feeds the contested obligations $W_t$ back to Retrieval as targeted queries.

\subsection{Appeal Refinement as a Cooperative-Adversarial Game}\label{ssec:game}
Each round constitutes a two-player refinement game between the Synthesis Agent $P$ (proponent) and the Payer Critic $K$ (opponent). $P$ seeks a draft that anchors every obligation on some satisfying frontier; $K$ seeks any obligation it can contest. The stage payoff is the deficiency $D(s_t)$: $P$ minimizes it, $K$ exposes it. Unlike unconstrained debate---where convergence is not guaranteed~\cite{du2023multiagent}---AGVF's game is played on the fixed, finite board $G_K$ with a monotone evidence store. This structure forces non-increasing deficiency, as we now show.

\subsection{Monotone Convergence}\label{ssec:conv}
We make the operating assumptions explicit; Section~6 discusses relaxing them.
\begin{assumption}[Static, complete policy graph]\label{as:static}
Within an episode, $\mathcal{I}$ produces $G_K$ once and it does not change; $\Phi_G$ captures the operative criteria.
\end{assumption}
\begin{assumption}[Monotone evidence]\label{as:mono}
Retrieval only adds verified facts, $F_{t+1}\supseteq F_t$; a verified fact is never retracted mid-episode.
\end{assumption}
\begin{assumption}[Sound and complete critique]\label{as:sound}
The Critic never contests an already-anchored obligation and does not hide open obligations: $W_t=\Delta(s_t)$, and therefore $v\in W_t \Rightarrow \mathrm{sat}_{F_t}(v)=0$.
\end{assumption}

\begin{lemma}[Monotone deficiency]\label{lem:mono}
Under Assumptions~\ref{as:static}--\ref{as:sound}, $D(s_{t+1})\le D(s_t)$ for all $t$.
\end{lemma}
\begin{proof}
By Assumption~\ref{as:static} the frontier collection $\mathcal{A}(G_K)$ is fixed. By Assumption~\ref{as:mono}, $F_t\subseteq F_{t+1}$, and since anchoring is inclusion-monotone---adding facts can turn a leaf's satisfaction from $0$ to $1$ but never from $1$ to $0$---we have $\mathrm{sat}_{F_t}(v)\le\mathrm{sat}_{F_{t+1}}(v)$ for every leaf $v$. Hence for every frontier $A$ the unmet cost $\sum_{v\in A,\,\mathrm{sat}=0} w_v$ is non-increasing in $t$, and so is its minimum over frontiers, which is exactly $D$.
\end{proof}

\begin{theorem}[Termination and fixed point]\label{thm:term}
Under Assumptions~\ref{as:static}--\ref{as:sound}, with integer-scaled weights and minimum weight $w_{\min}$, AGVF performs at most $N$ refinement rounds. The number of strictly deficiency-reducing rounds is at most $\lceil D(s_0)/w_{\min}\rceil$. If the run halts before the budget is exhausted, it terminates in a trajectory $\tau^{*}$ satisfying either \emph{(i)} $D(\tau^{*})=0$---the appeal anchors a complete satisfying frontier of the medical-necessity criterion---or \emph{(ii)} $D(\tau^{*})>0$ with the remaining open obligations unsupported by any admissible evidence returned by the configured evidence source.
\end{theorem}
\begin{proof}
Consider each round. If $W_t\neq\emptyset$ and Retrieval closes at least one obligation on the current cheapest frontier, then at least one such leaf flips $0\to1$, so by Lemma~\ref{lem:mono} $D$ strictly decreases by at least $w_{\min}$. Since $D$ is bounded below by $0$ and falls in quanta of at least $w_{\min}$, this can occur at most $\lceil D(s_0)/w_{\min}\rceil$ times. Otherwise either $W_t=\emptyset$ (the Critic exposes no open obligation, giving case (i) with $D=0$ under Assumption~\ref{as:sound}), or $W_t\neq\emptyset$ but Retrieval returns no new admissible fact for any contested obligation, yielding an evidentiary gap with respect to the configured evidence source (case (ii)). In all cases, the loop guard $t<N$ caps the number of refinement rounds by $N$.
\end{proof}

The dichotomy in Theorem~\ref{thm:term} is the theoretical crux of AGVF: the framework never closes a gap by fabricating evidence. Residual deficiency is surfaced to the Human-in-the-Loop reviewer as an honest, localized list of missing clinical documentation rather than being papered over---precisely the failure mode of the naive pipeline of \S1.2.

\subsection{Citation-Grounding Guarantee}\label{ssec:zero}
The unsupported-assertion constraint $\mathcal{C}_{\text{unsupported}}(\tau)=0$ of the CMDP is enforced not as a soft penalty but as a hard, deterministic citation-admissibility gate inside the Synthesis operator.

\begin{definition}[Citation-grounded assertion]
An assertion $c$ is \emph{citation-grounded} iff $\varphi(c)\neq\emptyset$ and every $f\in\varphi(c)$ lies in $\mathcal{F}_P$ (is a verified patient fact with a source locator). Define the citation-grounding violation count $C_{\text{cite}}(s)=|\{c\in\delta_s : c \text{ is not citation-grounded}\}|$.
\end{definition}

The Gated Synthesis operator applies a filter $\mathbb{G}$ that emits an assertion $c$ only if $c$ is citation-grounded; candidate assertions without admissible citations are dropped before the draft is formed. This is a syntactic admissibility guarantee. It should not be confused with full semantic entailment: a separate verifier is needed to prove that the wording of a cited assertion is completely entailed by the cited span.

\begin{lemma}[Citation-grounding invariant]\label{lem:ground}
For every state $s$ reachable under the gated Synthesis operator, $C_{\text{cite}}(s)=0$.
\end{lemma}
\begin{proof}
Immediate from the gate: $\delta_s$ contains only assertions that passed $\mathbb{G}$, each citation-grounded by definition, so the set of citation-ungrounded assertions is empty at every reachable state. Because $\mathbb{G}$ is a deterministic predicate over $\mathcal{F}_P$---not a probabilistic scorer---the guarantee holds with probability one, independent of the underlying LLM's sampling.
\end{proof}

Because Lemma~\ref{lem:ground} holds at every reachable state, it holds at the terminal $\tau^{*}$: $\mathcal{C}_{\text{unsupported}}(\tau^{*})=0$ for the implemented citation-admissibility predicate, discharging the corresponding CMDP constraint by construction rather than in expectation. Regulatory-template validation is a separate deterministic checklist interface in the reference implementation; it is not evaluated as a statutory compliance guarantee in Section~\ref{sec:validation}.

\subsection{Optimality Characterization}\label{ssec:opt}
Combining the above: among all trajectories that (a) respect the grounding invariant of Lemma~\ref{lem:ground} and (b) are reachable from $s_0$ under monotone retrieval over $\mathcal{F}_P$, $\tau^{*}$ attains the minimum achievable deficiency. We add one mild link between coverage and reward.
\begin{assumption}[Coverage monotonicity]\label{as:cov}
$P(\text{Overturn}\mid s)$ is non-decreasing as satisfied obligations on a viable frontier accrue; equivalently, the reward $\mathcal{R}(s)$ is non-increasing in $D(s)$.
\end{assumption}
Under Assumption~\ref{as:cov}, minimizing $D$ is a formal surrogate for the CMDP reward $\mathcal{R}(s)=P(\text{Overturn}\mid s)$. Hence $\tau^{*}$ is a constrained optimum of objective~(3) only with respect to the evidence made reachable by the configured retrieval process and the implemented admissibility gates. We do not claim that this proves real-world overturn probability. AGVF thereby converts an intractable stochastic optimization over free-form text into a monotone, finite, evidence-bounded search whose admissibility properties are explicit and auditable.

\section{System Implementation}\label{sec:impl}
This section describes how the theory of Section~\ref{sec:theory} becomes a running reference system. We focus on agent orchestration---how control passes between agents, how state is shared, and where the deterministic gates sit in the message flow. The design is vendor-neutral: AGVF can be instantiated on any LLM provider, retriever, and workflow runtime that satisfy the interface contracts below. The implementation evaluated in Section~\ref{sec:validation}, however, is intentionally narrower: hand-formalized policy graphs, synthetic oracle-labeled evidence pools, exact-key retrieval, a deterministic mock LLM, and a sound deterministic critic.

\subsection{Orchestration Model}
AGVF is coordinated by a thin orchestrator that owns the dossier state $s=(F_s,G_K,\delta_s)$ and advances it through the refinement loop of \S\ref{ssec:protocol}. The orchestrator holds minimal domain logic and instead sequences calls to the five agents, each exposed behind a narrow, typed interface. This separation keeps the convergence and citation-grounding invariants of \S\ref{ssec:conv}--\S\ref{ssec:zero} auditable, because every state transition is mediated by the orchestrator and recorded, rather than emerging from opaque agent-to-agent chatter.

\paragraph{Blackboard state.} Agents do not message one another directly; they operate through a shared, versioned dossier held by the orchestrator (a blackboard pattern). Each agent returns a typed artifact---verified facts, a deficiency set, a gated draft, or a weakness set---which the orchestrator merges. This makes the monotone-evidence assumption (Assumption~\ref{as:mono}) enforceable in code: the merge step is append-only for verified facts, so no agent can retract a previously grounded fact mid-episode.

\paragraph{Round controller.} A round controller implements the loop guard ``while $W_t\neq\emptyset$ and $t<N$.'' It maintains the iteration counter $t$ against the budget $N$, detects the fixed point (an empty weakness set, or a round that adds no new admissible facts), and on termination exposes the remaining open obligations for Human-in-the-Loop (HITL) review. The controller is the single place where the termination conditions of Theorem~\ref{thm:term} are operationalized.

\begin{figure}[tbp]
  \centering
  \includegraphics[width=0.86\linewidth]{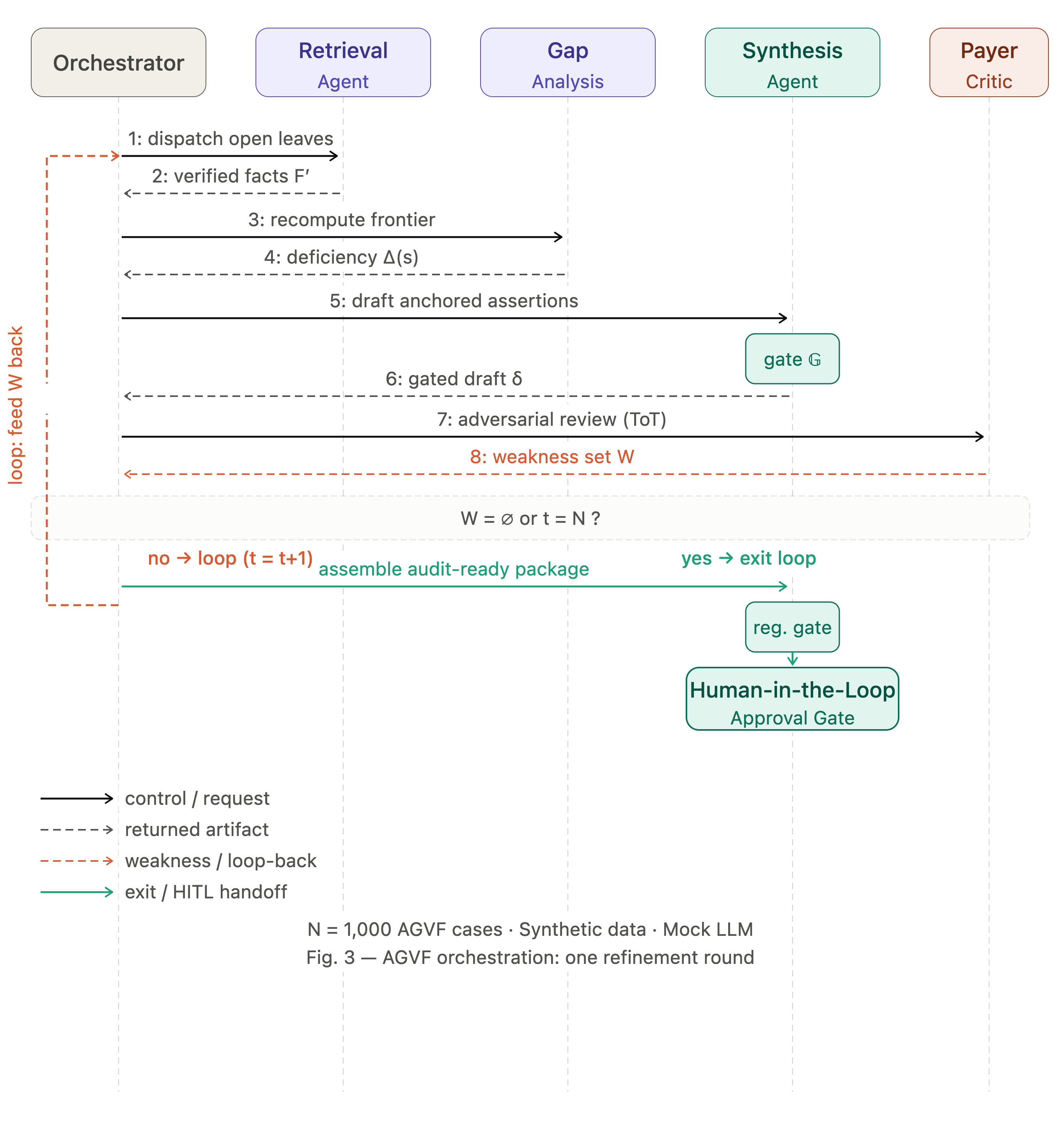}
  \caption{Message flow for a single AGVF refinement round. The orchestrator mediates every exchange; solid arrows are control/request messages and dashed arrows are returned artifacts. The deterministic citation gate $\mathbb{G}$ sits inside the Synthesis Agent, and an empty weakness set (or exhausting the budget $N$) exits the loop to HITL review.}
  \Description{Sequence-style diagram in which an orchestrator calls policy, retrieval, gap-analysis, synthesis, and critic agents, receives typed artifacts, and routes either another refinement round or a human review handoff.}
  \label{fig:orchestration}
\end{figure}

\subsection{Agent Interfaces and Interaction Contracts}
Each agent is a self-contained service with a single responsibility and an explicit pre/post-condition contract. The contracts are what let the orchestrator reason about progress without inspecting agent internals:
\begin{itemize}
  \item \textbf{Policy Ingestion} consumes a policy specification and emits the constraint graph $G_K$ with a stable identifier for every leaf predicate. Post-condition: $G_K$ is acyclic and every leaf carries a machine-checkable predicate. Invoked once per episode (Assumption~\ref{as:static}). The reference implementation loads hand-authored JSON graphs.
  \item \textbf{Retrieval} consumes the open-leaf set and the patient evidence pool and emits candidate facts with provenance pointers. Post-condition: every returned fact carries a source locator, so anchoring can be verified downstream. The evaluated synthetic mode uses exact-key matching; dense, sparse, hierarchical, or graph-traversal back-ends can satisfy the same interface in future real-LLM experiments.
  \item \textbf{Gap Analysis} consumes $G_K$ and current evidence and emits the cheapest-frontier deficiency set $\Delta(s)$. Post-condition: $\Delta(s)\subseteq$ open leaves, giving the orchestrator a monotone progress signal.
  \item \textbf{Synthesis} consumes $\Delta(s)$ and anchored evidence and emits a draft in which every assertion has passed the deterministic citation-grounding gate. Post-condition: the returned draft satisfies the citation-grounding invariant (Lemma~\ref{lem:ground}) before it leaves the agent boundary.
  \item \textbf{Payer Critic} consumes the gated draft and $G_K$ and emits a weakness set $W=\Delta(s)$ in the deterministic validation. For a learned or LLM-backed critic, the orchestrator must enforce soundness and completeness checks before the theorem applies.
\end{itemize}
Because each contract is checkable at the orchestrator boundary, a misbehaving agent (e.g., a Critic that violates soundness or hides open obligations) can be caught before it silently corrupts the loop---an operational safeguard for the assumptions the proofs rely on.

\subsection{Placement of the Deterministic Gates}
A central implementation decision is where the hard constraints live. AGVF places admissibility checks at agent boundaries, not in a post-hoc review pass. The citation-grounding gate $\mathbb{G}$ is embedded inside the Synthesis Agent, so assertions without admissible citations are never emitted into shared state. The reference implementation also exposes a regulatory-template checker, but Section~\ref{sec:validation} does not evaluate statutory compliance because no jurisdiction-specific checklist is populated. Situating the citation gate on the write path---rather than as an advisory score---is what makes the implemented citation-grounding constraint hold with probability one rather than in expectation.

\subsection{Failure Handling and Observability}
The orchestration layer treats agent calls as typed state transitions. Retrieval and critique steps are idempotent with respect to the dossier, so a production deployment can retry transient failures without double-counting evidence. The current reference implementation records an in-memory append-only trace keyed by dossier version: which agent ran, what summary delta it produced, and which policy obligations remain open. Durable trace persistence, source-span rendering, and operational HITL workflow integration are engineering extensions needed before deployment.

\subsection{Reference Implementation Fidelity}
Table~\ref{tab:fidelity} summarizes what the released artifact implements versus what remains an extensibility point.
\begin{table}[tbp]
  \centering
  \caption{Reference implementation fidelity relative to the AGVF architecture.}
  \label{tab:fidelity}
  \begin{tabular}{p{0.28\linewidth}p{0.31\linewidth}p{0.31\linewidth}}
    \toprule
    Component & Implemented in reported validation & Future production extension \\
    \midrule
    Policy graph & Hand-authored JSON Boolean graphs & Automated policy-to-graph extraction and expert validation \\
    Retrieval & Exact-key lookup over synthetic oracle-labeled pools & Dense, sparse, hierarchical, or graph retrieval with span verification \\
    Synthesis & Mock LLM text plus deterministic citation gate & Real LLM synthesis with semantic entailment verification \\
    Critic & Deterministic gap critic returning $\Delta(s)$ & Role-specialized LLM adversarial critic with boundary checks \\
    Regulatory checks & Checklist interface only & Jurisdiction-specific statutory template validation \\
    Audit trace & In-memory transition trace & Durable evidence-span audit package and HITL workflow \\
    \bottomrule
  \end{tabular}
\end{table}

\section{System Validation}\label{sec:validation}
This section validates the \emph{reference implementation} of AGVF: it confirms that the running system faithfully upholds the invariants proved in Section~\ref{sec:theory}. We emphasize at the outset what this section does and does not establish. The citation-grounding and monotone-deficiency properties hold \emph{by construction} (Lemmas~\ref{lem:mono}--\ref{lem:ground}, Theorem~\ref{thm:term}); the purpose of the study is therefore to verify that the implementation honors them and to quantify the effect of removing the mechanisms that enforce them. This is a validation of implementation fidelity on a controlled synthetic benchmark, not a demonstration of real-world efficacy---we return to that distinction in Section~\ref{ssec:threats}.

\subsection{Benchmark Design and Experimental Setup}
We evaluate on 1{,}000 synthetic appeal cases, 250 for each of four hand-formalized policy graphs spanning common administrative denial families: CO-50 General (6 leaf obligations), CO-197 Prior Authorization (9), CO-50 Surgical (10), and CO-50 DME (9). Each synthetic case contains an oracle-labeled evidence pool keyed to policy obligations, including missing-obligation cases and exclusion-blocked cases. Each case is processed by four arms: the full AGVF loop; a single-pass Retrieval-Augmented Generation baseline (RAG) that surfaces all available evidence at once; and two ablations---AGVF$-$Critic (critic removed) and AGVF$-$Gates (deterministic citation gate removed). To stress the citation gate, the AGVF$-$Gates arm injects two fabricated, unsupported assertions per case. All runs use a deterministic mock language model, a maximum of $N=8$ refinement rounds, and a retrieval reveal limit of one obligation per round; the seed is fixed at 42 for reproducibility.

\subsection{Data and Provenance}
No real patient records are used. Cases are generated synthetically; to make their clinical coding statistically realistic, the generator samples diagnosis codes, procedure codes, DRG, length-of-stay, and charge values from aggregate distributions extracted from the de-identified \emph{Hospital Inpatient Discharges (SPARCS De-Identified): 2024} dataset published by the New York State Department of Health. Only aggregate frequency distributions were used (profiled over 2{,}196{,}737 discharge records); no SPARCS row appears in the evaluation set, and no attempt was made to re-identify individuals or to link SPARCS with any other data. The denial framing, policy constraint graphs, clinical evidence, and obligation-anchoring are entirely synthetic. SPARCS parameterization affects only the surface realism of case attributes; it does not alter the obligation-satisfaction logic or the formal properties under test.

\subsection{Citation-Grounding}
Across all 1{,}000 AGVF cases the mean number of citation-ungrounded assertions is exactly $0.000$, and the citation-grounding violation rate is $0.0\%$---every emitted assertion cites at least one verified fact, as Lemma~\ref{lem:ground} requires (Figure~\ref{fig:grounding}). The RAG baseline and the AGVF$-$Critic arm likewise show no violations, because neither disables the gate. The AGVF$-$Gates arm, with the gate removed and two fabricated unsupported assertions injected per case, shows exactly $2.000$ citation-ungrounded assertions per case and a $100\%$ grounding-violation rate. The contrast is categorical rather than incremental: with the deterministic gate in place, citation-free or invalidly cited assertions cannot enter shared state; without it, every injected case is contaminated. This result confirms that the gate is load-bearing for the implemented admissibility guarantee.

\begin{figure}[tbp]
  \centering
  \includegraphics[width=\linewidth]{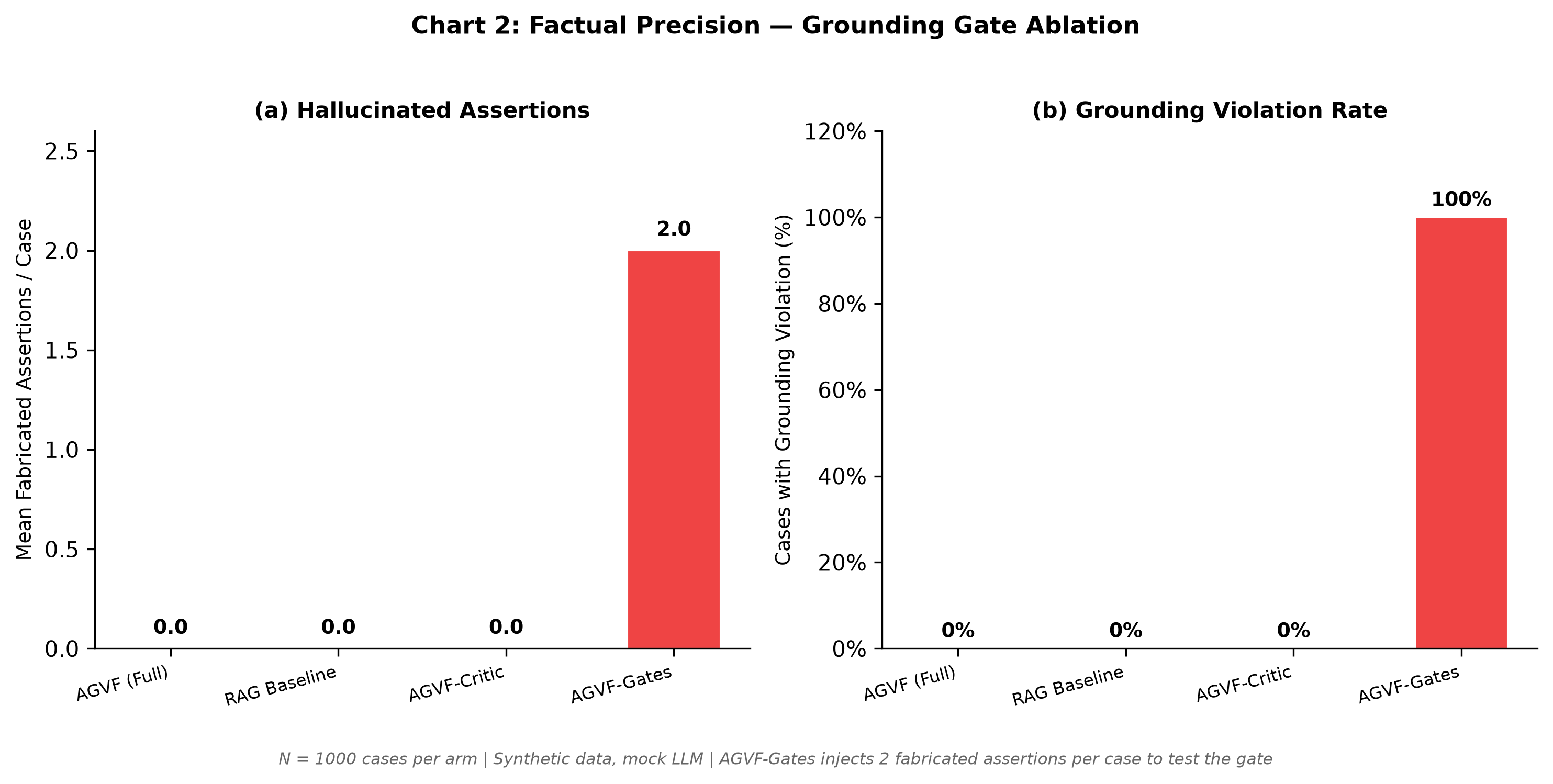}
  \caption{Citation-grounding gate ablation across 1{,}000 synthetic cases per arm. (a) mean fabricated unsupported assertions per case; (b) share of cases with any citation-grounding violation. AGVF, RAG, and AGVF$-$Critic show zero violations; removing the gate (AGVF$-$Gates) yields 2.0 fabricated assertions/case and a 100\% violation rate. Synthetic data, mock LLM.}
  \Description{Two-panel bar chart showing zero citation-grounding violations for AGVF, RAG, and AGVF without critic, and complete violation under the no-gate ablation with two fabricated assertions per case.}
  \label{fig:grounding}
\end{figure}

\subsection{Deficiency Convergence}
Figure~\ref{fig:convergence} traces the deficiency $D(s)$ across refinement rounds over all 1{,}000 AGVF cases. The mean deficiency falls monotonically from $11.1$ at initialization to $1.68$ at the final state, and---critically---$D(s)$ is non-increasing at every step of every case: the monotone-convergence rate is $100.0\%$, matching Theorem~\ref{thm:term} under the stated assumptions. The interquartile and 10th--90th-percentile bands narrow as rounds proceed, showing the behavior is consistent across cases and policies, not an averaging artifact. By contrast, the single-pass RAG baseline violates monotonicity in $4.0\%$ of cases (40 cases): by surfacing all evidence blindly in one pass, it can introduce exclusion evidence that drives $D$ from a finite value to infinity---a formally worse state in the benchmark. This illustrates that the invariant is a property of the disciplined refinement loop, not of retrieval in general.

\begin{figure}[tbp]
  \centering
  \includegraphics[width=\linewidth]{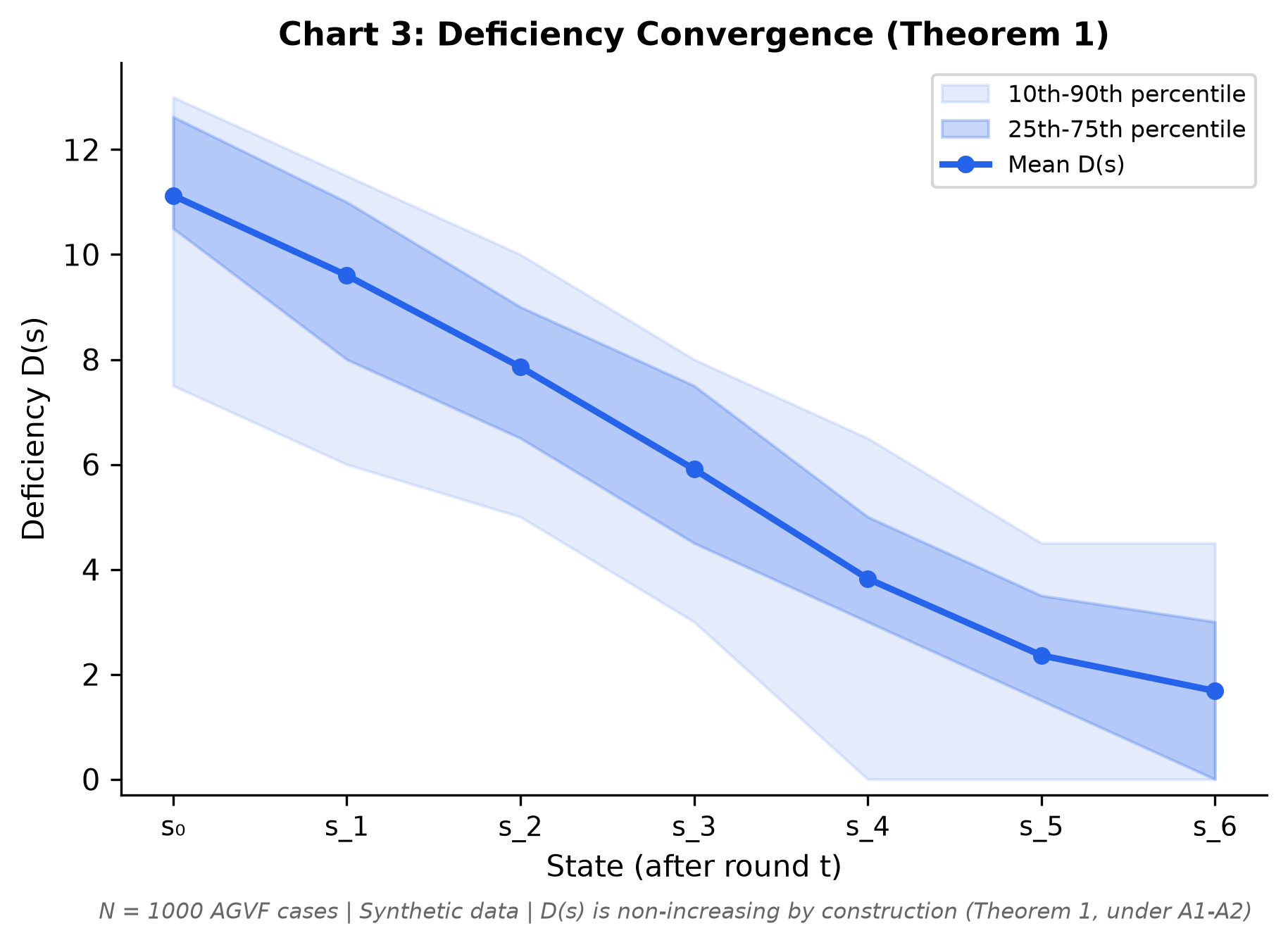}
  \caption{Deficiency $D(s)$ by refinement round over 1{,}000 AGVF cases, with 25th--75th and 10th--90th percentile bands. Mean $D$ falls $11.1 \rightarrow 1.68$; $D(s)$ is non-increasing in 100\% of cases (Theorem~\ref{thm:term}, under Assumptions~\ref{as:static}--\ref{as:sound}). Synthetic data.}
  \Description{Line chart with uncertainty bands showing mean deficiency decreasing monotonically over AGVF refinement rounds across the synthetic benchmark.}
  \label{fig:convergence}
\end{figure}

\subsection{Termination Behavior}
Figure~\ref{fig:termination} shows how episodes terminate. Under AGVF, $39.6\%$ of cases reach $D=0$ (a complete, fully anchored satisfying frontier) and the remaining $60.4\%$ terminate at a finite positive deficiency---an \emph{evidentiary gap}, surfaced to the human reviewer as a localized list of missing obligations. No AGVF case terminates in the infeasible $D=\infty$ state. This is the operational expression of Theorem~\ref{thm:term}'s dichotomy under the benchmark evidence source: the system either fully substantiates the criterion or reports what remains missing; it does not close a gap by fabrication. The RAG baseline, by contrast, produces 40 exclusion-blocked ($D=\infty$) cases, again reflecting the brittleness of undisciplined single-pass evidence surfacing in this synthetic setup.

\begin{figure}[tbp]
  \centering
  \includegraphics[width=\linewidth]{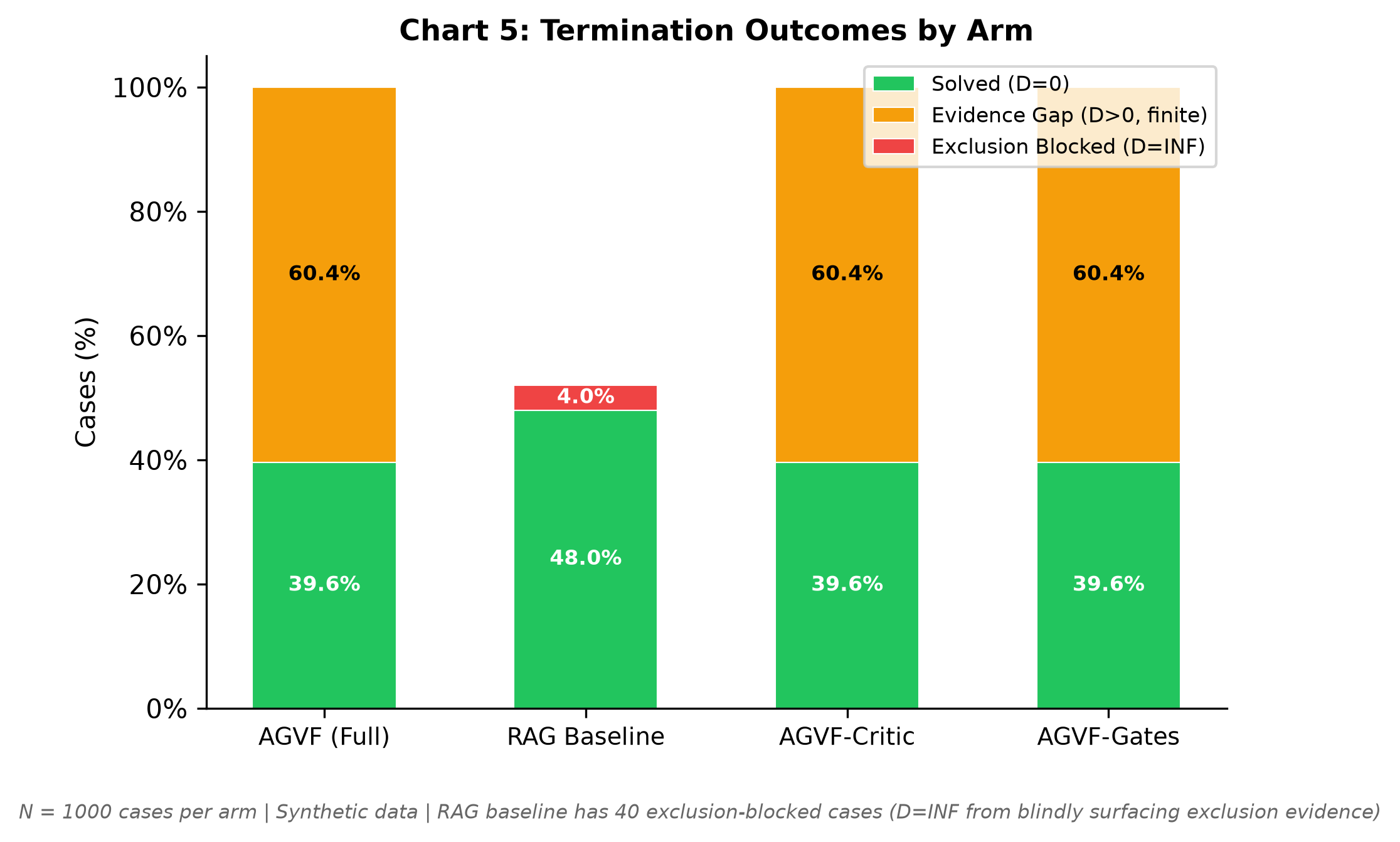}
  \caption{Termination outcomes per arm. AGVF: 39.6\% solved ($D=0$), 60.4\% finite evidence gaps, no infeasible states. RAG produces 40 exclusion-blocked ($D=\infty$) cases. Synthetic data.}
  \Description{Stacked bar chart comparing termination outcomes by arm, with AGVF split between solved and finite-gap cases and RAG including a small exclusion-blocked segment.}
  \label{fig:termination}
\end{figure}

\subsection{Ablation}
Figure~\ref{fig:ablation} isolates each component's contribution as the change relative to full AGVF. Removing the citation-grounding gate degrades the grounding-violation rate by $+100$ percentage points and adds $+2.0$ fabricated assertions per case, while leaving convergence and solved-rate untouched---the gate governs citation admissibility and nothing else, exactly as designed. Removing the critic produces no measurable change on this synthetic, mock-LLM configuration. We flag this plainly rather than obscure it: with a deterministic deficiency engine and a mock model, the critic's contested-obligation set coincides with the gap-analysis output, so the critic cannot be differentiated here. The critic's value is an empirical question that requires a real language model and semantically noisy cases; we defer that experiment to future work (Section~\ref{ssec:threats}).

\begin{figure}[tbp]
  \centering
  \includegraphics[width=\linewidth]{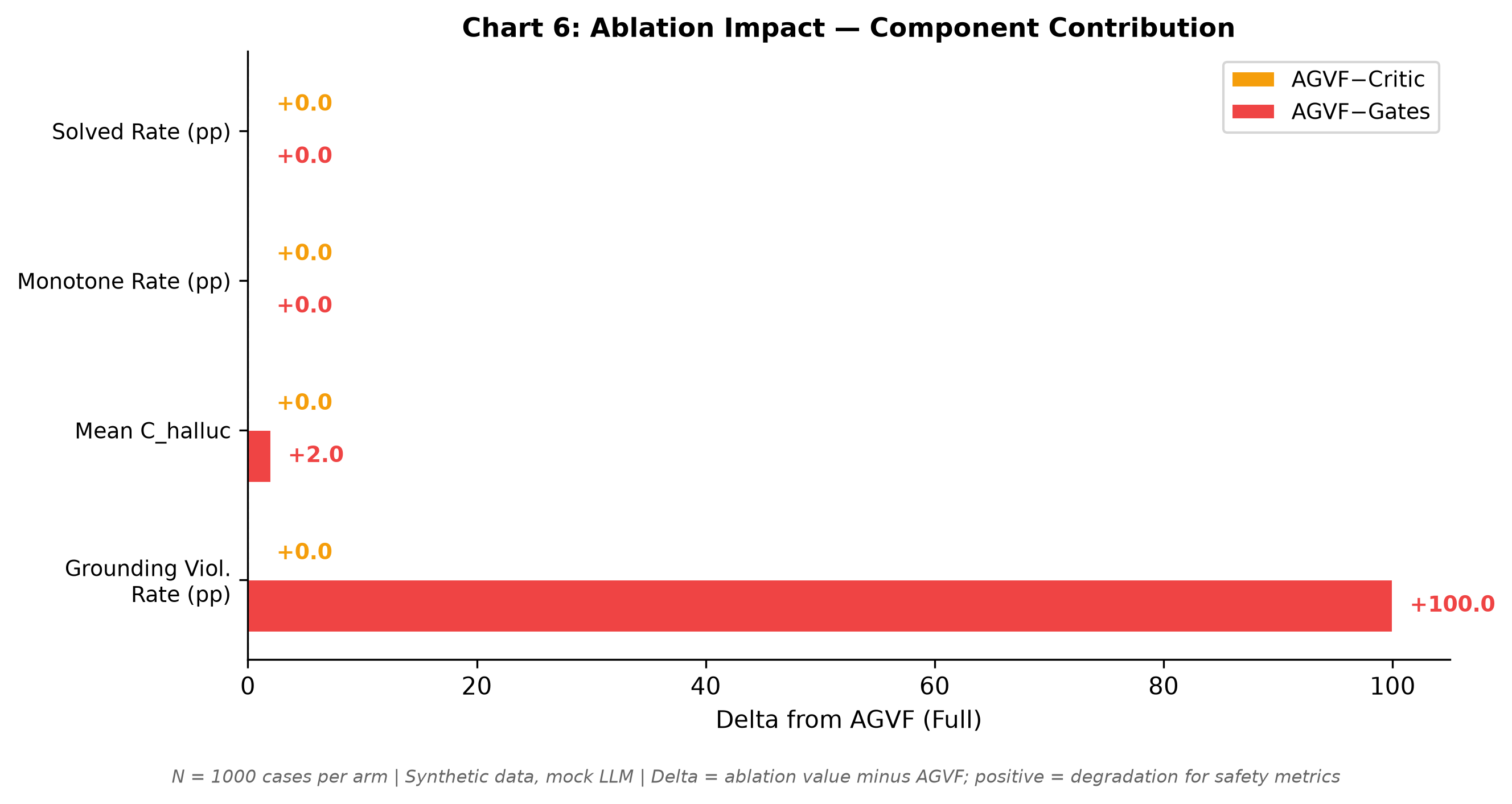}
  \caption{Component contribution as change from full AGVF. Gate removal: $+100$ pp citation-grounding violations, $+2.0$ fabricated assertions/case. Critic removal: no effect on this synthetic, mock-LLM configuration (expected---see \S\ref{sec:validation}, Ablation). Synthetic data.}
  \Description{Bar chart showing that removing the gate increases citation violations and fabricated assertions, while removing the critic has no measurable effect in the deterministic synthetic configuration.}
  \label{fig:ablation}
\end{figure}

\subsection{Threats to Validity}\label{ssec:threats}
We state the limitations directly, as they bound what the reader may conclude.

\noindent\textbf{Synthetic data.} No real denial cases or patient records are evaluated. Although attribute distributions are parameterized from aggregate de-identified discharge data, every case, denial framing, clinical evidence item, and obligation label is fabricated; the study therefore speaks to controlled system behavior, not clinical reality.

\noindent\textbf{LLM scope.} This study uses a deterministic stand-in rather than a production language model. The citation-grounding and convergence invariants are LLM-independent under the implemented gates, so they transfer at the control-logic level; but end-to-end LLM retrieval quality, adversarial critique, synthesis fluency, and semantic faithfulness are not fully exercised, and reported latencies reflect loop overhead rather than production model inference.

\noindent\textbf{Guarantees are constructive.} This section confirms the implementation is faithful to properties that are already proven. It does not, and cannot, show that AGVF improves appeal outcomes.

\noindent\textbf{Efficacy unmeasured.} We deliberately do not report an overturn-rate result. A credible efficacy measure requires clinician-adjudicated or payer-adjudicated outcomes on real de-identified cases, which we identify as the essential next study.

\noindent\textbf{Critic non-differentiable here.} As noted above, the synthetic/mock configuration cannot exercise the critic; its contribution remains an open empirical question for real LLM-backed experiments.

\noindent\textbf{Idealized assumptions.} The convergence guarantee relies on Assumptions~\ref{as:static} (a static, complete policy graph) and~\ref{as:sound} (a sound and complete critic/gap signal). Real policy ingestion is imperfect and a real critic may be unsound or incomplete; relaxing these is future work.

\noindent\textbf{Configuration-dependent descriptive metrics.} Solved-rate and round-count differences across arms are influenced by the one-obligation-per-round reveal limit and should be read as descriptive of loop behavior, not as efficacy comparisons.

\section{Conclusion and Future Work}
This paper presented AGVF, a multi-agent framework that turns healthcare appeal generation from unconstrained text production into constrained, evidence-bounded refinement. The contributions are mainly theoretical, architectural, and benchmark-oriented: appeal synthesis formalized as a Constrained Markov Decision Process over a policy constraint graph; a convergence result (Theorem~\ref{thm:term}) showing that refinement on that graph reduces deficiency monotonically and terminates in well-characterized states under explicit assumptions; and a construction (Lemma~\ref{lem:ground}) that prevents citation-ungrounded assertions from entering shared state. We accompanied the framework with a reference implementation and validated on 1{,}000 synthetic cases---parameterized from aggregate de-identified public discharge distributions---that the implementation upholds these invariants: zero citation-grounding violations across all AGVF cases, monotone convergence in every episode, and a deterministic gate that is the single load-bearing component for the implemented admissibility guarantee.

The central limitation of this work is equally its clearest direction forward: we validate that AGVF behaves as proven, but we do not establish that it improves real-world appeal outcomes. The guarantees are constructive, and the reported study uses synthetic data and a mock control-loop validation rather than real clinical adjudication. Three lines of future work follow directly.

\noindent\textbf{End-to-end live-LLM evaluation.} The immediate next study is to run a production language model for synthesis and critique inside the full AGVF loop over the same benchmark, adding noisy evidence, contradictory snippets, and realistic denial packets. The relevant metrics are unsupported-claim rate, citation validity, semantic faithfulness, critic-added corrections, residual deficiency, latency, and cost.

\noindent\textbf{Semantic verification.} The present gate verifies citation admissibility, not full textual entailment. A stronger verifier should map each generated claim to cited evidence spans and classify it as supported, contradicted, or not enough evidence. That extension would let the system report semantic unsupported-claim rates rather than only citation-grounding violations.

\noindent\textbf{Clinical efficacy evaluation.} A later study should evaluate real, de-identified denial cases with adjudicated outcomes---either payer decisions or blinded clinician ``would-overturn'' judgments---measuring appeal quality and overturn likelihood against manual and single-pass baselines. This is the evidence needed to move from ``the system is faithful to its guarantees'' to ``the system helps.'' We estimate a rigorously scored set of 50--150 cases per policy family as a credible starting scale.

\noindent\textbf{Relaxing the idealized assumptions.} The convergence guarantee rests on a static, complete policy graph (Assumption~\ref{as:static}) and a sound, complete critic/gap signal (Assumption~\ref{as:sound}). Real policy ingestion is imperfect and a learned critic may be unsound or incomplete. Extending the theory to dynamic graph discovery during refinement, and to bounded-error critics, would broaden AGVF's applicability while clarifying how the guarantees degrade under realistic conditions.

More broadly, AGVF illustrates a pattern that should apply beyond healthcare appeals: when an agentic LLM system must satisfy hard external constraints, encoding those constraints as a formal graph and enforcing admissibility through deterministic gates---instead of relying only on prompting or post-hoc review---turns an unbounded generation problem into a bounded, auditable one with explicit assumptions and testable invariants.

\section*{Availability}
A reference implementation of AGVF, including the synthetic case generator and evaluation harness, is available from the authors upon reasonable request.

\section*{Use of AI Tools}
The authors used AI-assisted tools during manuscript editing, code scaffolding, and artifact preparation. The named authors reviewed the manuscript, code, data-generation procedure, references, and reported results and take responsibility for the work.

\section*{Disclosure of Interests}
All authors are employed by Sliced Health, which develops commercial healthcare revenue-cycle software related to the subject of this paper. This work describes a research framework and reference implementation; the authors have no other competing interests to declare.

\bibliographystyle{ACM-Reference-Format}
\bibliography{references}

\end{document}